\documentclass{article}

\usepackage[utf8]{inputenc}
\usepackage[T1]{fontenc}
\usepackage{textcomp}
\usepackage{bm}

\usepackage{amsmath}
\usepackage{amssymb}
\usepackage{amsthm}
\usepackage{amsfonts}       

\usepackage{mathtools}
\DeclarePairedDelimiter{\prn}{(}{)}
\DeclarePairedDelimiter{\set}{\{}{\}}
\DeclarePairedDelimiterX{\Set}[2]{\{}{\}}{\,{#1}\;\delimsize|\;{#2}\,}
\DeclarePairedDelimiter{\abs}{|}{|}
\DeclarePairedDelimiter{\norm}{\|}{\|}
\DeclarePairedDelimiter{\inpr}{\langle}{\rangle}

\DeclarePairedDelimiter{\brc}{[}{]}
\DeclarePairedDelimiterX{\Brc}[2]{[}{]}{\,{#1}\;\delimsize|\;{#2}\,}

\usepackage[svgnames]{xcolor}
\usepackage{graphicx}

\usepackage{booktabs}
\usepackage{multirow}

\usepackage{fullpage}
\usepackage{newtxtext,newtxmath}
\usepackage{parskip}

\usepackage{etoolbox}
\cslet{blx@noerroretextools}\empty%
\usepackage[date=year,eprint=false,doi=false,isbn=false,backend=biber,giveninits=true,maxcitenames=2,maxbibnames=99,natbib=true,url=false,sorting=ynt,style=authoryear-comp,backref=true,uniquelist=false]{biblatex}

\DeclareNameAlias{author}{last-first}
\DeclareSourcemap{
    \maps[datatype=bibtex]{
        \map{
            \step[fieldset=editor, null]
            \step[fieldset=series, null]
            \step[fieldset=language, null]
            \step[fieldset=address, null]
            \step[fieldset=location, null]
            \step[fieldset=month, null]
        }
    }
}
\newrobustcmd{\MakeTitleCase}[1]{%
  \ifthenelse{\ifcurrentfield{title}}%
    {\MakeSentenceCase{#1}}%
    {#1}}
\DeclareFieldFormat[article,inbook,incollection,inproceedings,patent,thesis,unpublished]{titlecase}{\MakeTitleCase{#1}}
\DefineBibliographyStrings{english}{%
  backrefpage  = {cited on page}, 
  backrefpages = {cited on pages} 
}
\usepackage{algorithm}

\usepackage[colorlinks=true,citecolor=Navy,linkcolor=Maroon,urlcolor=Orchid,bookmarksnumbered,hypertexnames=false,pdfdisplaydoctitle,pdfusetitle,unicode]{hyperref}

\usepackage[noend]{algpseudocode}

\algnewcommand{\algorithmicinput}{\textbf{Input:}}
\algnewcommand{\Input}{\item[\algorithmicinput]}
\algnewcommand{\algorithmicoutput}{\textbf{Input:}}
\algnewcommand{\Output}{\item[\algorithmicoutput]}
\algnewcommand{\Break}{\textbf{break}}

\usepackage{url}
\usepackage{upref}
\usepackage[capitalize,noabbrev]{cleveref}

\crefname{step}{Step}{Steps}
\Crefname{step}{Step}{Steps}

\usepackage{autonum}

\usepackage{nicefrac}       
\usepackage{microtype}      
\usepackage{xspace}
\usepackage[color={red!100!green!33},colorinlistoftodos,prependcaption,textsize=small]{todonotes}
\usepackage{subcaption}
\usepackage{wrapfig}
\usepackage{siunitx}
\allowdisplaybreaks[4]
\usetikzlibrary{calc}
\usepackage{thm-restate}

\newtheorem{theorem}{Theorem}[section]
\newtheorem{lemma}[theorem]{Lemma}
\newtheorem{proposition}[theorem]{Proposition}
\newtheorem{corollary}[theorem]{Corollary}

\theoremstyle{definition}
\newtheorem{definition}[theorem]{Definition}

\newtheorem{remark}[theorem]{Remark}

\newcommand{\Z}{\mathbb{Z}}
\newcommand{\R}{\mathbb{R}}

\DeclareMathOperator{\argmin}{arg\,min}
\DeclareMathOperator{\argmax}{arg\,max}

\newcommand{\I}{I}

\newcommand{\Dcal}{\mathcal{D}}

\newcommand{\E}{\mathop{\mathbb{E}}}

\newcommand{\subloss}{{suboptimality loss}\xspace}
\newcommand{\sub}{{suboptimality}\xspace}

\newcommand{\subandestloss}{{suboptimality and estimate losses}\xspace}
\newcommand{\estloss}{{estimate loss}\xspace}
\newcommand{\totloss}{{total loss}\xspace}

\title{Revisiting Online Learning Approach to Inverse Linear Optimization: A Fenchel--Young Loss Perspective and Gap-Dependent Regret Analysis}

\author{%
Shinsaku Sakaue\\
The University of Tokyo and RIKEN AIP\\
\href{mailto:sakaue@mist.i.u-tokyo.ac.jp}{sakaue@mist.i.u-tokyo.ac.jp}
\and
Han Bao\\
Kyoto University\\
\href{mailto:bao@i.kyoto-u.ac.jp}{bao@i.kyoto-u.ac.jp}
\and
Taira Tsuchiya\\
The University of Tokyo and RIKEN AIP\\
\href{mailto:tsuchiya@mist.i.u-tokyo.ac.jp}{tsuchiya@mist.i.u-tokyo.ac.jp}
}
\date{}

\begin{document}

\maketitle

\begin{abstract}
  This paper revisits the online learning approach to inverse linear optimization studied by \citet{Barmann2017-wl}, where the goal is to infer an unknown linear objective function of an agent from sequential observations of the agent's input-output pairs. 
  First, we provide a simple understanding of the online learning approach through its connection to online convex optimization of \emph{Fenchel--Young losses}. 
  As a byproduct, we present an offline guarantee on the \emph{suboptimality loss}, which measures how well predicted objective vectors explain the agent's choices, without assuming the optimality of the agent's choices. 
  Second, assuming that there is a gap between optimal and suboptimal objective values in the agent's decision problems, we obtain an upper bound independent of the time horizon $T$ on the sum of suboptimality and \emph{estimate losses}, where the latter measures the quality of solutions recommended by predicted objective vectors. 
  Interestingly, our gap-dependent analysis achieves a faster rate than the standard $O(\sqrt{T})$ regret bound by exploiting structures specific to inverse linear optimization, even though neither the loss functions nor their domains possess desirable properties, such as strong convexity.\looseness=-1
\end{abstract}

\section{Introduction}\label{sec:introduction}
Linear optimization is arguably the most widely used model of decision-making.
Inverse linear optimization is its inverse problem, where the goal is to infer linear objective functions from observed outcomes.
Since the early development in geographical science \citep{Tarantola1988-tq,Burton1992-dc}, inverse linear optimization has been an important subject of study \citep{Ahuja2001-cv,Heuberger2004-zv,Chan2019-zg} and used in various applications, ranging from route recommendation to healthcare \citep{Chan2023-qk}.

Inverse linear optimization is particularly interesting when forward linear optimization is a decision-making model of a human \emph{agent}.\footnote{
  An agent is sometimes called an expert, but we here use ``agent'' to avoid confusion with experts in online learning.
} 
Then, the linear objective function represents the agent's internal preference.
If the agent repeatedly takes an action upon facing a set of feasible actions, inverse linear optimization can be seen as online learning of the agent's internal preference from pairs of the feasible sets and the agent's choices. 
\citet{Barmann2017-wl} studied this setting and proposed an elegant approach based on online learning, which is the focus of this paper and is described below.

Consider an agent who addresses decision problems of the following linear-optimization form for $t = 1,\dots,T$:
\begin{equation}\label{eq:decision-problem}
  \mathrm{maximize}\;\; \inpr{c^*, x} 
  \qquad 
  \mathrm{subject~to}\;\; x \in X_t,
\end{equation}
where $c^* \in \R^n$ ($n \in \Z_{>0}$) is the agent's objective vector, and $X_t \subseteq \R^n$ is the set of actions the agent can take in round $t$, which is not necessarily convex.
Let $x_t \in X_t$ be the agent's choice in round~$t$.
Informally, we want to learn $c^*$ from the input-output pairs, $\set*{(X_t, x_t)}_{t=1}^T$.

To this end, we need learning objectives that quantify how good prediction $\hat c$ of $c^*$ is based on the observations, $\set*{(X_t, x_t)}_{t=1}^T$.
Let $\Theta \subseteq \R^n$ be the space of objective vectors, from which we pick prediction $\hat c \in \Theta$.
\citet{Barmann2017-wl} considered the following learning objective:
\begin{equation}\label{eq:barmann-objective}
  \begin{aligned}
    \max\Set*{\inpr{\hat c, x'}}{x' \in X_t} - \inpr{\hat c, x_t}
    &&
    t \in \set{1,\dots,T}. 
  \end{aligned}
\end{equation}
If (non-zero) $\hat c$ makes \eqref{eq:barmann-objective} zero for all $t=1,\dots,T$, then $\hat c$ consistently explains why the agent takes $x_t \in X_t$. 
Since~\eqref{eq:barmann-objective} represents the \sub of $x_t$ for $\hat c$, it is called \emph{\subloss} \citep{Mohajerin-Esfahani2018-jf,Chen2020-kc,Sun2023-nc}.
For any $\hat c \in \Theta$, letting $\hat x_t \in \argmax_{x'\in X_t}\inpr{\hat c, x'}$, we~can write the \subloss \eqref{eq:barmann-objective} as follows:
\begin{equation}\label{eq:suboptloss}
  \begin{aligned}
    \ell^{\mathrm{sub}}_t(\hat c) \coloneqq \inpr{\hat c, \hat x_t - x_t} 
    &&
    t \in \set{1,\dots,T}.
  \end{aligned}
\end{equation}

Relying solely on the \subloss is sometimes insufficient, as the zero \subloss is attained by trivial all-zero prediction $\hat c = 0$ (cf.\ \citealt[Section~6.1]{Mishra2024-nk}).\footnote{\citet{Barmann2017-wl} avoided this issue by considering $\Theta$ that does not contain the all-zero vector---specifically, the probability simplex $\Theta = \Set{c \in \R^n_{\ge0}}{\norm{c}_1 = 1}$.\label{footnote:all-zero}}
Another natural learning objective free from this issue is the so-called \emph{\estloss} \citep{Chen2020-kc,Sun2023-nc},\footnote{These studies also consider the prediction loss, $\norm{\hat c - c}^2$, which we cannot use as no direct supervision on $\hat c$ is given.} defined by 
\begin{equation}\label{eq:estimate-loss}
  \begin{aligned}
    \ell^{\mathrm{est}}_t(\hat c)
    \coloneqq
    \inpr{c^*, x_t - \hat x_t} 
    &&
    t \in \set{1,\dots,T}.
  \end{aligned}
\end{equation}
Ideally, we want to find $\hat c \in \Theta$ that attains $\ell^{\mathrm{sub}}_t(\hat c) = \ell^{\mathrm{est}}_t(\hat c) = 0$ for $t=1,\dots,T$.
\citet{Barmann2017-wl} showed that a proxy goal is achievable in an online setting. 
Specifically, for each round $t$, given $\set*{(X_i, x_i)}_{i=1}^{t-1}$, we can compute prediction $\hat c_t \in \Theta$ so that
\begin{equation}\label{eq:subopt-est-loss}
  \sum_{t=1}^T
  \prn[\big]{
    \underbrace{\ell^{\mathrm{sub}}_t(\hat c_t)
    +
    \ell^{\mathrm{est}}_t(\hat c_t)}_{\text{\tiny Total loss}}
    }
  =
  \sum_{t=1}^T
  \inpr{\hat c_t - c^*, \hat x_t - x_t}
\end{equation}
grows at the rate of $O(\sqrt{T})$, and hence the average of $\ell^{\mathrm{sub}}_t(\hat c_t) +\ell^{\mathrm{est}}_t(\hat c_t)$ converges to zero as $T\to\infty$. 
For convenience, we call the sum of the \subandestloss, $\ell^{\mathrm{sub}}_t(\hat c_t) +\ell^{\mathrm{est}}_t(\hat c_t)$ in \eqref{eq:subopt-est-loss}, the \emph{\totloss}. 
The idea of \citet{Barmann2017-wl} is based on online learning, which works as follows: for each $t = 1,\dots,T$, an adversary chooses a convex loss $f_t\colon\Theta\to\R$ and a learner computes $\hat c_t \in \Theta$ based on past of observations $\set{f_i}_{i=1}^{t-1}$ so that the \emph{regret} $\sum_{t=1}^T \prn*{f_t(\hat c_t) - f_t(c^*)}$ grows only sublinearly in $T$ for any $c^* \in \Theta$.
\citet{Barmann2017-wl} regarded $f_t\colon\hat c \mapsto \inpr{\hat x_t - x_t, \hat c}$ as a linear loss function by considering $\hat x_t$ independent of~$\hat c$,\footnote{
  Taking the dependence of $\hat x_t$ on $\hat c$ into account, $f_t$ coincides with the non-linear suboptimality loss \eqref{eq:suboptloss}. 
  The above linear-loss perspective is explicitly described in \citet[Section~3.3]{Barmann2020-hh}, an extended arXiv version of \citealt{Barmann2017-wl} that includes additional results, such as an online-gradient-descent approach and an offline guarantee.
} 
and used the well-known multiplicative weight update method (MWU); consequently, the $O(\sqrt{T})$ regret bound of MWU translates to the bound on~\eqref{eq:subopt-est-loss}.

\subsection{Our Contributions}
We revisit the online learning approach of \citet{Barmann2017-wl} and make the following contributions. 

\textbf{Drawing connection to OCO of Fenchel--Young losses.}
While \citet{Barmann2017-wl} applied the idea of online learning to the linear loss $f_t\colon\hat c \mapsto \inpr{\hat x_t - x_t, \hat c}$, we view the problem as online convex optimization (OCO) of \emph{Fenchel--Young losses} \citep{Blondel2020-tu}. 
Specifically, we first show that the \subloss~\eqref{eq:suboptloss} is a particular type of Fenchel--Young loss, allowing us to take advantage of its useful properties (see \cref{prop:fyloss-properties}). 
Then, we show that the \totloss in~\eqref{eq:subopt-est-loss} appears as the \emph{linearized regret} of the Fenchel--Young loss, to which we apply the Follow-The-Regularized-Leader (FTRL) and obtain a bound of $O(\sqrt{T})$ that applies to various situations. 
We thus offer a simple understanding of the online learning approach to inverse linear optimization from the viewpoint of Fenchel--Young losses. 
As a byproduct, we obtain an offline guarantee on the \subloss that enjoys broader applicability than a previous result in \citet{Barmann2020-hh}.\looseness=-1

\textbf{Gap-dependent regret analysis.}
We then discuss going beyond the $O(\sqrt{T})$ bound on the \totloss over~$T$ rounds~\eqref{eq:subopt-est-loss}.
We focus on the case where there is a gap between the optimal and suboptimal objective values in the agent's decision problems~\eqref{eq:decision-problem}; we parameterize this by $\Delta > 0$.
Under this assumption, we obtain a bound of $O(1/\Delta^{2})$, which is independent of the time horizon $T$.
Unlike typical gap assumptions in online learning, our assumption is imposed on the agent's decision problems, highlighting that problem structures specific to inverse linear optimization can help accelerate online learning of the agent's objective vector. 
We discuss situations where the gap assumption is reasonable in \cref{subsec:applications}.\looseness=-1

\subsection{Related Work}

\textbf{Inverse optimization.}
Most studies on inverse optimization are based on the KKT condition \citep{Iyengar2005-fl,Tan2020-lp,Mishra2024-nk}.
The online learning approach of \citet{Barmann2017-wl} enjoys broader applicability as it does not rely on specific representations of the agent's feasible regions, allowing for even non-connected sets (see \cref{subsec:problem-setting}).
\citet{Besbes2021-ak,Besbes2023-zm} recently obtained an $O(n^4\ln T)$ regret bound with respect to $\ell^\mathrm{est}_t$, where $n$ is the dimension of the ambient space of objective vectors. 
More recently, \citet{Sakaue2025-xa} obtained an improved bound of $O(n\ln T)$ on $\ell^\mathrm{est}_t + \ell^\mathrm{sub}_t$. 
On the other hand, our bound in \cref{sec:fast-rate} is independent of $T$ under the gap assumption.
Other online-learning-based approaches with different criteria are also widely studied \citep{Dong2018-ap,Chen2020-kc,Sun2023-nc}.
A minor remark is that a criterion used by \citet{Chen2020-kc}, called the simple loss, is different from the \totloss~\eqref{eq:subopt-est-loss}, although they might look similar.
Indeed, their simple loss is defined as a linear function \citep[Lemma~2]{Chen2020-kc}, while the \totloss~\eqref{eq:subopt-est-loss} consists of non-linear terms, such as $\ell^\mathrm{sub}_t$.\looseness=-1

\textbf{Smart predict, then optimize.}
\emph{Smart predict, then optimize} \citep{Elmachtoub2022-og} is a relevant framework for learning objective functions. 
Similar settings are studied under the name of \emph{decision-focused learning} \citep{Wilder2019-yw} and \emph{contextual linear optimization} \citep{Hu2022-yq}.
These frameworks are intended for learning models that predict objective functions from contextual information. 
To this end, they require past objective functions
as training data, which is not available in our inverse optimization setting.

\textbf{Fenchel--Young loss.}
The Fenchel--Young-loss framework \citep{Blondel2020-tu} offers a general recipe for designing convex surrogate losses in supervised learning. 
Its usefulness in a remotely related scenario---learning models that generate objective functions based on contextual information and optimal solutions---was demonstrated by \citet{Berthet2020-wm}.
Different from their work, we elucidate an unexpected connection between the Fenchel--Young loss and the \subandestloss in inverse linear optimization, which is of interest in its own right. 
\citet{Sakaue2024-mu} used Fenchel--Young losses for online structured prediction, focusing on a different criterion called the \emph{surrogate regret}.

\textbf{Online convex optimization.}
Our work is intended to benefit from rich theory of OCO \citep{orabona2023modern} in inverse linear optimization.
In OCO, going beyond the $O(\sqrt{T})$ regret requires additional assumptions, such as the strong convexity and exp-concavity of loss functions \citep{Hazan2007-ta} and curved decision sets \citep{Huang2017-xs}.
Also, gap conditions related to the \emph{learner's} choices sometimes help achieve fast rates in stochastic settings \citep{Lai1985-cf,Auer2002-ad}. 
Different from them, our gap condition, detailed in \cref{def:gap}, is imposed on the \emph{agent's} decision problems~\eqref{eq:decision-problem} and hence specific to inverse optimization.

\section{Preliminaries}
\textbf{Notation.}
Let $\inpr{\cdot, \cdot}$ be the standard inner product on $\R^n$, or the paring on the primal space $\R^n$ of agent's actions and the dual space $\Theta \subseteq \R^n$ of predictions.
Let $\norm{\cdot}$ denote a norm on the agent's space $\R^n$ and $\norm{\cdot}_\star$ the dual norm on $\Theta$. 
We sometimes omit the subscript of $t$ when the dependence on $t$ is irrelevant.
For example, we use $(X, x)$ to represent a pair of agent's feasible set and output, $\hat c \in \Theta$ a prediction, and $\hat x \in \argmax_{x' \in X} \inpr{\hat c, x'}$ an optimal solution in terms of prediction $\hat c$.\footnote{As $\hat x$ depends on $\hat c$, it should be read like $\hat x(\hat c)$. However, we omit the dependence on $\hat c$ for brevity. We use the ``hat'' accent on $x$ to represent the dependence on some prediction $\hat c$.}
Similarly, let $\ell^\mathrm{sub}$ and $\ell^\mathrm{est}$ denote the \subandestloss, respectively, without specifying the round, $t$.\looseness=-1

\subsection{Problem Setting}\label{subsec:problem-setting}
As described in \cref{sec:introduction}, we aim to learn an unknown objective vector $c^* \in \R^n$ of an agent who addresses problem~\eqref{eq:decision-problem} for $t = 1,\dots,T$. 
We assume that every feasible set $X_t \subseteq \R^n$ is a non-empty compact set containing the agent's choice $x_t$. 
Until \cref{sec:fast-rate}, $x_t \in X_t$ is not necessarily optimal for $c^*$, unlike \citet{Barmann2017-wl,Barmann2020-hh}.
Also, $X_t$ is not necessarily convex; the only requirement about $X_t$ is that we can compute an optimal solution $x \in \argmax_{x' \in X_t} \inpr{c, x'}$ for any $c \in \R^n$.
A typical tractable case is when $X_t$ is a polytope specified by linear constraints; then, computing an optimal solution reduces to a linear program.
Even if $X_t$ is non-convex (and even non-connected, as in the case of integer linear programs), we can use empirically efficient solvers, such as Gurobi, to find an optimal solution.
In \cref{sec:fast-rate}, we will additionally assume that the agent's choices are optimal, i.e., $x_t \in \argmax_{x'\in X_t}\inpr{c^*, x'}$, and that the agent's decision problems~\eqref{eq:decision-problem} satisfy a certain gap condition, which we explain later.

We assume that the space of predictions, $\Theta \subseteq \R^n$, is a closed convex set containing the agent's true objective~$c^*$.
Additionally, while not strictly necessary, it is common to assume that $\Theta$ does not contain the origin; otherwise, trivial all-zero prediction $\hat c = 0$ attains $\ell^\mathrm{sub}_t(0) = 0$, as discussed in \cref{sec:introduction}.
For example, \citet{Barmann2017-wl} studied the case where $\Theta = \Set{c \in \R^n_{\ge0}}{\norm{c}_1 = 1}$, i.e., the probability simplex.
While our analysis does not use $0 \notin \Theta$, we suppose that this condition holds in \cref{subsec:online-to-batch}, as we focus on the suboptimality loss therein.

Similar to \citet{Barmann2017-wl}, we mainly focus on the online setting, which goes as follows. 
For each $t=1,\dots,T$, we compute $\hat c_t \in \Theta$ using past observations $\set*{(X_i, x_i)}_{i=1}^{t-1}$. 
Then, we observe $(X_t, x_t)$ and incur the $t$th total loss $\ell^{\mathrm{sub}}_t(\hat c_t) + \ell^{\mathrm{est}}_t(\hat c_t)$. 
We aim to minimize the \totloss over $T$ rounds~\eqref{eq:subopt-est-loss}.
In \cref{subsec:online-to-batch}, we discuss a stochastic offline setting and give a bound on the \subloss in expectation, which we detail later.

\subsection{Fenchel--Young Loss}
We briefly describe the basics of the Fenchel--Young loss \citep{Blondel2020-tu}, which is defined as follows. 

\begin{definition}\label{def:fy-loss}
  Let $\Omega\colon\R^n\to \R\cup\set{+\infty}$ be a (possibly non-convex) lower semi-continuous function. 
  For any given $x \in \R^n$ with $\Omega(x) < +\infty$, the Fenchel--Young loss is defined as a function of $c \in \R^n$ as follows: 
  \[
    L_\Omega(c; x) = \Omega^*(c) + \Omega(x) - \inpr{c, x},
  \]
  where $\Omega^*(c) \coloneqq \sup_{x \in \R^n} \inpr{c, x} - \Omega(x)$ is the convex conjugate of $\Omega$.
\end{definition}

Many common loss functions in machine learning, such as the squared and logistic losses, can be represented as Fenchel--Young losses by choosing  appropriate $\Omega$.
We can interpret $L_\Omega(c; x)$ as the discrepancy between primal and dual variables, $x$ and $c$, measured by the \emph{Fenchel--Young inequality}, $\Omega^*(c)+ \Omega(x) \ge \inpr{c, x}$, a fundamental inequality in convex analysis that readily follows from the definition of $\Omega^*$.
In our case, we want to measure the quality of (dual) prediction $\hat c$ by using the agent's output $x$ as (primal) feedback.
Therefore, the Fenchel--Young loss can be useful to measure the quality of prediction $\hat c$ in inverse linear optimization.

We then present some properties of the Fenchel--Young loss that will be useful in the subsequent discussion.
We refer the reader to \citet[Proposition~2]{Blondel2020-tu} for more information.

\begin{proposition}\label{prop:fyloss-properties}
  For any $x \in \R^n$, the Fenchel--Young loss $L_\Omega(\cdot; x)$ is non-negative and convex. 
  For any $\hat c \in \R^n$, if $\hat x \in \argmax_{x' \in \R^n}{\inpr{\hat c, x'} - \Omega(x')}$ exists, then the residual vector is a subgradient, i.e., $\hat x - x \in \partial L_\Omega(\hat c; x)$.
\end{proposition}
We emphasize that \cref{prop:fyloss-properties} holds even if $\Omega$ is non-convex (and even if non-continuous). 
These properties are not difficult to confirm: 
the non-negativity is due to the Fenchel--Young inequality, and the convexity follows from the fact that $\Omega^*(c)$ is the supremum of linear functions $c \mapsto \inpr{c, x} - \Omega(x)$.
The subgradient property follows from Danskin's theorem \citep{Danskin1966-gb}, which ensures $\hat x \in \partial \Omega^*(\hat c)$.

\section{Fenchel--Young-Loss Perspective}\label{sec:fy-loss-perspective}
This section describes a Fenchel--Young-loss perspective on inverse linear optimization.
Below, let $X \subseteq \R^n$ and $x \in X$ be the agent's feasible set and output, respectively. 
For convenience, let $L_X(\cdot; x)$ denote the Fenchel--Young loss with $\Omega = \I_X$,\footnote{When $X$ is a convex set, this Fenchel--Young loss is also called the structured-perceptron loss \citep{Blondel2020-tu}.} where $\I_X\colon\R^n\to\set{0, +\infty}$ is the indicator function of $X$, i.e., $\I_X(x') = 0$ if $x' \in X$ and $+\infty$ otherwise.
Since $X$ is a non-empty compact set, an optimal solution to $\max_{x'\in X} \inpr{c, x'}$ exists for any $c \in \R^n$; therefore, $L_X(\cdot; x)$ satisfies \cref{prop:fyloss-properties}.

\subsection{$\ell^\mathrm{sub}$ is a Fenchel--Young Loss}
First, we observe that for any agent's choice $x \in X$, $L_X(\cdot; x)$ coincides with the \subloss \eqref{eq:suboptloss}, $\ell^{\mathrm{sub}}$.
\begin{proposition}\label{prop:subloss-fyloss}
  The \subloss $\ell^{\mathrm{sub}}$ \eqref{eq:suboptloss} is a Fenchel--Young loss with $\Omega = I_X$, i.e., for any $\hat c \in \R^n$, 
  \begin{equation}
    L_X(\hat c; x) 
    = 
    \max_{x'\in X}\inpr{\hat c, x'}-\inpr{\hat c, x}
    = 
    \inpr{\hat c, \hat x - x}
    =
    \ell^{\mathrm{sub}}(\hat c),
  \end{equation}  
  where $\hat x \in \argmax_{x'\in X}\inpr{\hat c, x'}$.
\end{proposition}
\begin{proof}
  From \cref{def:fy-loss} with $\Omega = I_X$ 
  and 
  $\Omega^*(\hat c) = \sup_{x' \in \R^n}{\inpr{\hat c, x'} - \Omega(x')} = \max_{x' \in X}\inpr{\hat c, x'}$, 
  it holds that 
  $
    L_X(\hat c; x) 
    = 
    \Omega^*(\hat c) + \Omega(x) - \inpr{\hat c, x}
    = 
    \max_{x' \in X}\inpr{\hat c, x'} + 0 - \inpr{\hat c, x} 
    = 
    \inpr{\hat c, \hat x - x} 
    =
    \ell^{\mathrm{sub}}(\hat c)
  $.
\end{proof}
This connection, although easy to derive, clarifies that the \subloss $\ell^{\mathrm{sub}}$ enjoys the properties of Fenchel--Young losses in \cref{prop:fyloss-properties}.
Specifically, $\ell^{\mathrm{sub}}$ is non-negative and convex, and the residual $\hat x - x$ is a subgradient of $\ell^{\mathrm{sub}}(\hat c)$. 
\begin{remark}\label{rem:realizability}
  While these properties are also proved in \citet[Proposition~3.1]{Barmann2020-hh} by assuming the agent's optimality $x \in \argmax_{x'\in X}\inpr{c^*, x'}$, the above Fenchel--Young-loss perspective offers a comprehensive understanding of the \subloss through convex analysis.
  It is also worth mentioning that the Fenchel--Young loss, $L_X(\cdot; x)$, can be defined for any $x \in X$, regardless of its optimality for $c^*$.
\end{remark}

\subsection{$\ell^\mathrm{sub} + \ell^\mathrm{est} =$ the Linearized Regret of a Fenchel--Young Loss}
In online learning, \emph{linearization} is a common technique for reducing OCO to online linear optimization (OLO) \citep[Section~2.3]{orabona2023modern}. 
For any convex losses $f_t\colon\Theta \to \R$, predictions $\hat c_t \in \Theta$, subgradients $g_t \in \partial f_t(\hat c_t)$, and comparator $c^* \in \Theta$, the \emph{linearized regret} is defined by $\sum_{t=1}^T \inpr{g_t, \hat c_t - c^*}$.
While this definition depends on the choice of subgradient $g_t \in \partial f_t(\hat c_t)$, the convexity of $f_t$ ensures $f_t(\hat c_t) - f_t(c^*) \le \inpr{g_t, \hat c_t - c^*}$ for any $c^* \in \Theta$ and $g_t \in \partial f_t(\hat c_t)$.
Thus, any upper bound on $\sum_{t=1}^T \inpr{g_t, \hat c_t - c^*}$ achieved with an OLO method applies to the original regret, $\sum_{t=1}^T \prn*{f_t(\hat c_t) - f_t(c^*)}$.\footnote{One might wonder whether the learner can compete against $g_t \in \partial f_t(\hat c_t)$, which is selected in reaction to $\hat c_t$. Fortunately, online learning methods are designed to work with such a reactive adversary. 
This fact is also used in \citet{Barmann2017-wl} when applying MWU to $f_t\colon \hat c \mapsto \inpr{\hat x_t - x_t, \hat c}$, where $\hat x_t \in \argmax_{x'\in X_t}\inpr{\hat c_t, x'}$ depends on $\hat c_t$.}

The following \cref{prop:linearized-regret} says that if we apply the linearization technique to the regret with respect to the Fenchel--Young loss, $L_X(\hat c; x) -  L_X(c^*; x)$, the resulting linearized regret equals the \totloss, $\ell^\mathrm{sub} + \ell^\mathrm{est}$, in \eqref{eq:subopt-est-loss}. 
\begin{proposition}\label{prop:linearized-regret}
  Let $\hat c \in \Theta$, $\hat x \in \argmax_{x' \in X}\inpr{\hat c, x'}$, $x \in X$, and $g = \hat x - x$. 
  For any $c^* \in \R^n$, $\inpr{g, \hat c - c^*}$ is the linearized regret of $L_X(\hat c; x) -  L_X(c^*; x)$, which satisfies
  \begin{equation}
    L_X(\hat c; x) -  L_X(c^*; x)
    \le 
    \inpr{g, \hat c - c^*} 
    =
    \ell^{\mathrm{sub}}(\hat c)
    +
    \ell^{\mathrm{est}}(\hat c),
  \end{equation}
  where $\ell^{\mathrm{sub}}(\hat c) = \inpr{\hat c, \hat x - x}$ and $\ell^{\mathrm{est}}(\hat c) = \inpr{c^*, x - \hat x}$.
\end{proposition}
\begin{proof}
  From \cref{prop:fyloss-properties}, $L_X(\cdot; x)$ is convex and $g = \hat x - x \in \partial L_X(\hat c; x)$ holds. 
  Thus, $L_X(\hat c; x) -  L_X(c^*; x)$ is upper bounded by $\inpr{g, \hat c - c^*}$.
  Plugging $g = \hat x - x$ into this yields the equality.\looseness=-1
\end{proof}

We have observed how the \sub and total losses relate to the Fenchel--Young loss and its linearized regret, which bridges between inverse linear optimization and OCO.
Built on this, \cref{sec:bounds,sec:fast-rate} provide theoretical guarantees for inverse linear optimization.

\subsection{Discussion on a Difference from the Previous Understanding}
We briefly digress to discuss a conceptual difference from \citet{Barmann2017-wl}, who derived the online learning approach to inverse linear optimization by viewing $f_t\colon \hat c \mapsto \inpr{\hat x_t - x_t, \hat c}$ as a \emph{linear loss}, as mentioned in \cref{sec:introduction}.
Actually, once we apply the linearization technique to the regret with respect to the \subloss $\ell^\mathrm{sub}_t$ as in \cref{prop:linearized-regret}, the resulting linearized regret coincides with the regret with respect to the linear loss, i.e., $\inpr{g_t, \hat c_t - c^*} = \inpr{\hat x_t - x_t, \hat c_t - c^*} = f_t(\hat c_t) - f_t(c^*)$. 
Hence, our approach recovers the same result as that of \citet{Barmann2017-wl}, as detailed in \cref{subsec:omd}.
Still, the above Fenchel--Young-loss perspective offers yet another way to understand how the online learning approach works. 
Below, we discuss a subtle related difference, which suggests that it is more natural to interpret the online learning approach as OCO of the \emph{nonlinear} \sub losses, $\ell^\mathrm{sub}_t$.

Let us discuss the same setting as that of \citet{Barmann2017-wl}, where $\Theta$ is the probability simplex and each loss $f_t$ is viewed as the linear loss.
When considering the regret of predictions $\hat c_1,\dots,\hat c_T \in \Theta$, it is usual to compare them with an \emph{ex-post} optimal prediction, $c^*_{\min} \in \Theta$, that minimizes the cumulative loss $\sum_{t=1}^T f_t(c)$ in hindsight.
Since $\Theta$ is the probability simplex and $f_t$'s are linear, $c^*_{\min}$ is a standard basis vector (unless $\sum_{t=1}^T \prn*{\hat x_t - x_t}$ is normal to a face of~$\Theta$). 
Such $c^*_{\min}$ may not be expressive enough to represent the human agent's preference $c^*$; this suggests that the linear-loss perspective encounters a dissonance between the ex-post optimal prediction $c^*_{\min}$ and the agent's preference $c^*$. 
By contrast, our formulation uses the suboptimality loss~$\ell^\mathrm{sub}_t$ as the target loss, and reduces OCO of~$\ell^\mathrm{sub}_t$ to OLO of $f_t$ via linearization. 
Therein, $\ell^\mathrm{sub}_t$ is nonlinear, and the corresponding ex-post optimal prediction $c^*_{\min}$, which minimizes $\sum_{t=1}^T\ell^\mathrm{sub}_t(c)$, would be a more plausible representation of the agent's preference~$c^*$.
However, we note that this difference does not affect regret bounds discussed below.

\section{Theoretical Guarantees for General Cases}\label{sec:bounds}
We discuss the online setting in \cref{subsec:omd} and present a bound on the linearized regret, or the \totloss over $T$ rounds \eqref{eq:subopt-est-loss}, which we denote by $R_T$ for convenience:
\begin{equation}\label{eq:linearized-regret}
  R_T \coloneqq \sum_{t=1}^T \inpr{g_t, \hat c_t - c^*} = 
  \sum_{t=1}^T \prn*{\ell^\mathrm{sub}_t(\hat c_t) + \ell^\mathrm{est}_t(\hat c_t)}.
\end{equation}
Here, $\hat c_t$ is the $t$th prediction, $\hat x_t \in \argmax_{x' \in X_t} \inpr{\hat c_t, x'}$ is an optimal solution for $\hat c_t$, and $g_t = \hat x_t - x_t$ is a subgradient of $L_{X_t}(\hat c_t; x_t) = \ell^\mathrm{sub}_t(\hat c_t)$.
In \cref{subsec:online-to-batch}, we discuss a stochastic offline setting and present a bound on the \subloss in expectation.

\subsection{Online Guarantee on $\ell^\mathrm{sub} + \ell^\mathrm{est}$}\label{subsec:omd}

\begin{algorithm}[tb]
	\caption{FTRL for inverse linear optimization}
	\label{alg:omd}
	\begin{algorithmic}[1]
    \For{$t=1,\dots,T$}
    \State Output $\hat c_t \in \argmin_{c \in \Theta} \beta_t\psi(c) + \sum_{i=1}^{t-1}\inpr{g_i, c}$. 
    \Comment{$\hat c_1 \in \Theta$ is arbitrary, and $\hat c_t = \hat c_{t-1}$ if $g_{t-1} = 0$.}
    \State Observe $(X_t, x_t)$.
    \State Compute $\hat x_t \in \argmax_{x' \in X_t} \inpr{\hat c_t, x'}$.
    \State Set $g_t = \hat x_t - x_t$. \Comment{$g_t \in \partial L_{X_t}(\hat c_t; x_t)$.}
    \EndFor
	\end{algorithmic}
\end{algorithm}

We consider the online setting and give a bound on $R_T$.
We use the well-known FTRL described in \citet[Section~7]{orabona2023modern}, which is shown in \cref{alg:omd} for our setting for completeness.
Let $\lambda > 0$ and $\psi\colon\Theta\to\R$ be a differentiable $\lambda$-strongly convex function with respect to the dual norm $\norm{\cdot}_\star$. 
The following bound on $R_T$ is a consequence of the existing analysis of FTRL.
For completeness, we provide the proof in \cref{asec:proof}.
\begin{proposition}[{cf.~\citealt[Section~7]{orabona2023modern}}]\label{prop:omd-regret-bound}
  Assume that there exists $B > 0$ such that
  \[
    \max\set[\Big]{
      2^{5/2}\lambda\max_{c, c' \in \Theta}\norm{c - c'}_\star^2, 
      \max_{c, c' \in \Theta} \prn*{\psi(c) - \psi(c')}
    }
    \le B^2.
  \]
  Let $\hat c_1,\dots,\hat c_T \in \Theta$ be the outputs of \cref{alg:omd} with 
  $
    \beta_t = \frac{2^{1/4}}{B}\sqrt{\frac{\sum_{i=1}^{t-1} \norm{g_i}^2}{\lambda}}
  $.
  For any $c^* \in \Theta$, the linearized regret~\eqref{eq:linearized-regret} (equivalently, the \totloss over $T$ rounds~\eqref{eq:subopt-est-loss}) is bounded as follows:
  \begin{equation}\label{eq:omd-regret}
    R_T = \sum_{t=1}^T \inpr{g_t, \hat c_t - c^*}
    \le 
    2^{5/4}B\sqrt{\frac{1}{\lambda}\sum_{t=1}^{T}\norm{g_t}^2}.
  \end{equation}  
  In particular, if $\norm{g_t} = \norm{\hat x_t - x_t} \le K$ for $t=1,\dots,T$ for some $K > 0$, we have $R_T \le 2^{5/4}KB\sqrt{T/\lambda}$.
\end{proposition}
\textbf{Recovering \citet[Theorem~3.3]{Barmann2017-wl}.}
As a concrete example, we show that \cref{prop:omd-regret-bound} implies \citet[Theorem~3.3]{Barmann2017-wl} 
as a special case.
Here, we set $\norm{\cdot} = \norm{\cdot}_\infty$ and $\norm{\cdot}_\star = \norm{\cdot}_1$. 
In their setting, $\Theta = \Set{c \in \R^n_{\ge0}}{\norm{c}_1 = 1}$ is the probability simplex, and the $\ell_\infty$-diameter of the agent's feasible sets, $X_1,\dots,X_T$, is bounded by some $K > 0$; 
therefore, $\max_{c, c' \in \Theta}\norm{c - c'}_1^2 \le 4$ and $\norm{g_t}_\infty = \norm{\hat x_t - x_t}_\infty \le K$ ($t=1,\dots,T$) hold.
Regarding $\psi$ in FTRL, we use the negative Shannon entropy, $\Theta \ni c \mapsto \inpr{c, \ln c}$, where $\ln$ is taken element-wise. 
From Pinsker's inequality, $\psi$ is $1$-strongly convex with respect to $\norm{\cdot}_1$, hence $\lambda = 1$.
From $\max_{c, c' \in \Theta} \prn*{\psi(c) - \psi(c')} \le \ln n$, it suffices to set $B = 2^{11/4}\sqrt{\ln n}$ (for $n \ge 2$). 
Thus, \cref{prop:omd-regret-bound} implies
\begin{equation}\label{eq:barmann-regret}
  \sum_{t=1}^T \inpr{\hat c_t - c^*, \hat x_t - x_t} = R_T \le 16K\sqrt{T\ln n}, 
\end{equation}
achieving a bound of $O(K\sqrt{T\ln n})$ on the total loss over $T$ rounds~\eqref{eq:subopt-est-loss}, as with \citet[Theorem~3.3]{Barmann2017-wl}.\footnote{
  The constant is increased compared to \citet[Theorem~3.3]{Barmann2017-wl} as a trade-off for making \eqref{eq:omd-regret} depend on $\sum_{t=1}^T\norm{g_t}^2$, which is crucial in the gap-dependent analysis in \cref{sec:fast-rate}.
  We can recover exactly the same bound by using FTRL with a different choice of $\beta_t$ (see \cref{asec:barmann}).
  }
In addition, the online learning methods are essentially identical: 
\citet{Barmann2017-wl} used MWU, or the exponentiated gradient method, which is a special case of FTRL with $\psi(c) = \inpr{c, \ln c}$ \citep[Section~7.5]{orabona2023modern}.

\Cref{prop:omd-regret-bound} also applies to various other settings with different convex sets $\Theta$ and convex functions $\psi$. 
For example, it is not difficult to recover a similar bound to \citet[Theorem~3.11]{Barmann2020-hh} achieved with the online subgradient descent method (cf.\ \citealt[Example~7.11]{orabona2023modern}). 
We will also use \cref{prop:omd-regret-bound} in the gap-dependent analysis in \cref{sec:fast-rate}.\looseness=-1

It is also worth noting that \cref{prop:omd-regret-bound} applies to the regret with respect to the \subloss thanks to \cref{prop:subloss-fyloss,prop:linearized-regret}.
The following \cref{cor:subopt}, while straightforward, highlights our viewpoint that the online learning approach can be seen as OCO of $\ell^\mathrm{sub}$.
\begin{corollary}\label{cor:subopt}
  It holds that 
  \begin{equation}\label{eq:subopt-regret}
    R^\mathrm{sub}_T \coloneqq\sum_{t=1}^T \prn*{\ell^\mathrm{sub}_t(\hat c_t) - \ell^\mathrm{sub}_t(c^*)} \le R_T,
  \end{equation} 
  where $R^\mathrm{sub}_T$ is the regret with respect to the \subloss, and hence \cref{prop:omd-regret-bound} also applies to $R^\mathrm{sub}_T$.  
\end{corollary}

\subsection{Offline Guarantee on $\ell^\mathrm{sub}$}\label{subsec:online-to-batch}
We then show an offline guarantee on the \subloss~\eqref{eq:suboptloss}.
Having established the connection to the Fenchel--Young loss in \cref{sec:fy-loss-perspective}, the following \cref{thm:online-to-batch} is immediate from the standard \emph{online-to-batch conversion} \citep{Cesa-Bianchi2004-id}.
Still, it is worth mentioning due to its wider applicability than the previous result by \citet[Theorem~3.14]{Barmann2020-hh}, as detailed later.

As with \citet[Section~3.4]{Barmann2020-hh}, we assume that observations $(X, x)$ follow an unknown distribution $\Dcal$ and that $0 \notin \Theta$ holds, as discussed in \cref{subsec:problem-setting}.
The following \cref{thm:online-to-batch} states that any bound on the regret with respect to the \subloss, $R^\mathrm{sub}_T$, (and hence any bound on $R_T$) translates into a bound on the expected \subloss as follows.

\begin{restatable}{theorem}{onlinetobatch}\label{thm:online-to-batch}
  Assume that $(X_1,x_1),\dots,(X_T,x_T)$ follow i.i.d.\ some distribution $\Dcal$. 
  Compute $\hat c_1,\dots,\hat c_T \in \Theta$ by using any OCO algorithm such that $R^\mathrm{sub}_T$ in \eqref{eq:subopt-regret} is bounded. 
  Then, for any $c^* \in \Theta$, the average prediction, $\hat c = \frac1T \sum_{t=1}^T \hat c_t$, satisfies\looseness=-1 
  \[
  \E\brc*{
      \ell^\mathrm{sub}_{X, x}(\hat c)
  }
  \le 
  \E_{(X, x) \sim \Dcal}\brc*{
    \ell^\mathrm{sub}_{X, x}(c^*)
  }
  +
  \E\brc*{\frac{R^\mathrm{sub}_T}{T}}, 
  \]
  where the expectation on the left-hand side is taken over $\set*{(X_t,x_t)}_{t=1}^T\sim \Dcal^T$ and $(X, x) \sim \Dcal$.
  Specifically, if we use \cref{alg:omd} and obtain $R_T = O(KB\sqrt{T/\lambda})$ as in \cref{prop:omd-regret-bound}, it holds that 
  \[
  \E\brc*{
      \ell^\mathrm{sub}_{X, x}(\hat c)
  }
  - 
  \E_{(X, x) \sim \Dcal}\brc*{
    \ell^\mathrm{sub}_{X, x}(c^*)
  }
  = 
  O\prn*{\frac{KB}{\sqrt{T\lambda}}}.
  \]
\end{restatable}
We defer the proof to \cref{asec:online-to-batch}, as it directly follows from the standard online-to-batch conversion applied to the regret with respect to the convex \subloss.

  The previous offline result \citep[Theorem~3.14]{Barmann2020-hh} focuses on the case of $\E_{(X, x) \sim \Dcal}\brc*{
    \ell^\mathrm{sub}_{X, x}(c^*)
  } = 0$, which requires that every $(X, x)\sim\Dcal$ arising with non-zero probability satisfies $x \in \argmax_{x'\in X}\inpr{c^*, x'}$.
  This condition is sometimes restrictive as agents may not take best actions at all times (e.g., \citealt{Jagadeesan2021-xa}). 
  By contrast, \cref{thm:online-to-batch} holds even when $\E_{(X, x) \sim \Dcal}\brc*{
    \ell^\mathrm{sub}_{X, x}(c^*)
  } > 0$, which quantifies the \sub of the agent's choices.  
  This extended applicability stems from the fact that the Fenchel--Young loss, $L_X(\cdot, x)$, can be defined for any $x \in X$, as discussed in \cref{rem:realizability}. 
  The fact that the offline guarantee with the broader applicability is easily derived from standard arguments in OCO suggests the merit of viewing inverse linear optimization as OCO of Fenchel--Young losses.

  Also, while \cref{thm:online-to-batch} pertains to the average prediction $\hat c$, we can use \emph{anytime} online-to-batch conversion \citep{Cutkosky2019-ik} to ensure the last-iterate convergence.
  Specifically, by using any OLO algorithm, we can obtain iterates $\hat c_t$'s such that subgradients are evaluated at $\hat c_t$'s and individual iterate $\hat c_t$ enjoys the offline guarantee.

\section{Gap-Dependent Bound on Total Loss}\label{sec:fast-rate}
We return to the online setting and consider bounding the linearized regret $R_T$ \eqref{eq:linearized-regret}, or the \totloss over~$T$ rounds.
In \cref{subsec:omd}, we have observed that the bound of $O(\sqrt{T})$ is achievable, as is often the case in OCO.
In general OCO (or even in OLO), the rate of $O(\sqrt{T})$ is tight \citep[Chapter~5]{orabona2023modern}, and improving this requires additional assumptions, such as the strong convexity or exp-concavity of loss functions \citep{Hazan2007-ta}. 
Unfortunately, the Fenchel--Young loss, $L_X(\cdot; x)$, does not enjoy such properties.\footnote{In general, Fenchel--Young loss $L_\Omega$ is strongly convex if $\Omega$ is smooth. In our case with $\Omega = \I_X$, $L_\Omega$ is typically piecewise linear and do not enjoy those properties.}
Given this, an interesting question is: 
\emph{can we improve the bound of $O(\sqrt{T})$ by exploiting structures specific to inverse linear optimization?}
Below, we show that we can achieve a bound independent of $T$ if the agent's decision problems satisfy a certain gap condition.

First, we formally define our assumption on the gap between optimal and suboptimal objective values.
\begin{definition}[$\Delta$-gap condition]\label{def:gap}
  Let $c^* \in \Theta$ be the agent's objective vector.
  We say the agent's decision problem \eqref{eq:decision-problem}, which is specified by $c^* \in \Theta$ and $X_t \subseteq \R^n$, satisfies the \emph{$\Delta$-gap condition} for $\Delta > 0$ if, for every $t = 1,\dots,T$, it holds that
  \begin{equation}
    \inpr{c^*, x - \hat x} \ge \Delta\norm{x - \hat x} 
    \quad
    \text{$\forall\hat x \in X_t$}    
  \end{equation}
  for $x = \argmax_{x' \in X_t}\inpr{c^*, x'}$.
\end{definition}
The above condition implies that $x$ is the unique optimal solution for $c^*$; otherwise, there exists $\hat x \in X_t$ with $\inpr{c^*, x - \hat x} = 0$ and $\norm{x - \hat x} > 0$, violating the inequality.
In \cref{subsec:applications}, we will discuss situations where the $\Delta$-gap condition is reasonable.
Note that the condition is imposed on the agent's decision problems and does not make any explicit assumptions on the loss functions or their domain $\Theta$.

The main result of this section is \cref{thm:fast-rate}, which offers a bound of $O(1/\Delta^2)$ on the linearized regret~\eqref{eq:linearized-regret}, or the \totloss over $T$ rounds, under the $\Delta$-gap condition.

\begin{theorem}\label{thm:fast-rate}
  Assume the same conditions as \cref{prop:omd-regret-bound}, i.e., 
  $\psi$ is $\lambda$-strongly convex with respect to $\norm{\cdot}_\star$, 
  $2^{5/2}\lambda\max_{c, c' \in \Theta}\norm{c - c'}_\star^2$ and $\max_{c, c' \in \Theta} \prn*{\psi(c) - \psi(c')}$ are bounded by $B^2$ from above, and the diameter of $X_t$'s with respect to $\norm{\cdot}$ is at most $K$. 
  Additionally, assume the following two conditions: 
  for every $t=1,\dots,T$, 
  (i) the agent's decision problem~\eqref{eq:decision-problem} satisfies the $\Delta$-gap condition, and 
  (ii) the agent's choice is optimal for $c^*$, i.e., $x_t \in \argmax_{x'\in X_t}\inpr{c^*, x'}$.
  Then, \cref{alg:omd} achieves the following bound on the linearized regret $R_T$~\eqref{eq:linearized-regret}:
  \[
    R_T \le \frac{2^{5/4}KB^3}{\lambda^{3/2}\Delta^2}.
  \]
\end{theorem}
Notably, the above regret bound is achieved by the same FTRL as that used in \cref{prop:omd-regret-bound} without knowing $\Delta$. 
That is, FTRL automatically adapts to the gap in the agent's decision problems. 
It should also be noted that, unlike the results presented so far, the agent's outputs $x_t$'s are assumed to be optimal for $c^*$, which is common in the literature: 
\citet{Besbes2023-zm} crucially used this condition to derive the $O(\ln T)$ bound on the \estloss;
\citet{Barmann2017-wl} also assumed the condition to obtain the $O(\sqrt{T})$ bound on the \totloss, although it is unnecessary as we have seen in \cref{sec:bounds}.\looseness=-1

\subsection{Proof of \texorpdfstring{\cref{thm:fast-rate}}{Theorem~\ref{thm:fast-rate}}}
Before proving \cref{thm:fast-rate}, we present a useful inequality derived from the $\Delta$-gap condition.

\begin{lemma}\label{lem:uniform-convexity-inequality}
  Assume that the same conditions as those in \cref{thm:fast-rate} hold.
  For any $t \in \set{1,\dots,T}$, $\hat c_t \in \Theta$, and $\hat x_t \in \argmax_{x \in X_t}\inpr{\hat c_t, x}$, it holds that 
  \begin{equation}\label{eq:uniform-convexity-inequality}
    \norm{x_t - \hat x_t}^2 
    \le 
    \frac{KB}{2^{5/4}\sqrt{\lambda}\Delta^2}
    \inpr{c^* - \hat c_t, x_t - \hat x_t}.  
  \end{equation}
\end{lemma}
\begin{proof}
  Since $\hat x_t$ and $x_t$ are optimal for $\hat c_t$ and $c^*$, respectively, we have 
  \begin{equation}
    \begin{aligned}
      0 \le \inpr{\hat c_t, \hat x_t - x_t}
      && 
      \text{and}
      &&
      \Delta\norm{x_t - \hat x_t} 
      \le
      \inpr{c^*, x_t - \hat x_t},
    \end{aligned}
  \end{equation}
  where the latter is due to the $\Delta$-gap condition.
  Adding these inequalities and squaring both sides yield
  \[
    \Delta^2\norm{x_t - \hat x_t}^2
    \le 
    \inpr{c^* - \hat c_t, x_t - \hat x_t}^2.
  \]
  Since $\inpr{c^* - \hat c_t, x_t - \hat x_t}$ is non-negative, the right-hand side is upper bounded by $\frac{KB}{2^{5/4}\sqrt{\lambda}}\inpr{c^* - \hat c_t, x_t - \hat x_t}$ due to $\norm{c^* - \hat c_t}_\star^2 \le \frac{B^2}{2^{5/2}\lambda}$ and $\norm{x_t - \hat x_t} \le K$, obtaining the desired inequality.
\end{proof}

Now we are ready to prove \cref{thm:fast-rate}.
\begin{proof}[Proof of \cref{thm:fast-rate}]
  We assume that the linearized regret $R_T = \sum_{t=1}^T \inpr{\hat x_t - x_t, \hat c_t - c^*}$ is positive; otherwise, it is trivially upper bounded by $0$.
  Recall $g_t = \hat x_t - x_t$ for $t=1,\dots,T$. 
  From \eqref{eq:uniform-convexity-inequality} in \cref{lem:uniform-convexity-inequality}, we have 
  \begin{equation}
    \norm{g_t}^2 
    \le \frac{KB}{2^{5/4}\sqrt{\lambda}\Delta^2} \inpr{g_t, \hat c_t - c^*}.
  \end{equation}
  Summing over $t=1,\dots,T$, we obtain 
  \begin{equation}\label{eq:regret-lb}
    \sum_{t=1}^T\norm{g_t}^2 
    \le 
    \frac{KB}{2^{5/4}\sqrt{\lambda}\Delta^2} \sum_{t=1}^T \inpr{g_t, \hat c_t - c^*} 
    = 
    \frac{KB}{2^{5/4}\sqrt{\lambda}\Delta^2} R_T.
  \end{equation}
  Combining this with the bound \eqref{eq:omd-regret} on $R_T$ in \cref{prop:omd-regret-bound}, we obtain 
  \begin{equation}
    R_T \le \frac{2^{5/8}K^{1/2}B^{3/2}}{\lambda^{3/4}\Delta} \sqrt{R_T}
    \Leftrightarrow 
    R_T \le \frac{2^{5/4}KB^3}{\lambda^{3/2}\Delta^2}, 
  \end{equation}
  thus completing the proof.
\end{proof}
The above proof also utilizes the Fenchel--Young loss perspective at a high level. 
Observe that the proof is built on two inequalities: 
the regret bound~\eqref{eq:omd-regret} that depends on $\sum_{t=1}^T\norm{g_t}^2$ (or a bound achieved by \emph{adaptive} algorithms \citep[Section~4.2]{orabona2023modern}), namely, 
\[
  \sum_{t=1}^T \prn*{
    \ell^\mathrm{sub}_t(\hat c_t) - \ell^\mathrm{sub}_t(c^*)
  } \le \sum_{t=1}^T \inpr{g_t, \hat c_t - c^*} \lesssim \sqrt{\sum_{t=1}^T\norm{g_t}^2},
\] 
and \eqref{eq:uniform-convexity-inequality} in \cref{lem:uniform-convexity-inequality}, which, roughly speaking, ensures
\[
  \norm{\hat x - x_t}^2 \lesssim \inpr{\hat x_t - x_t, \hat c_t - c^*}/\Delta^2.
\]
What bridges these two inequalities is the ``subgradient as a residual'' property of the Fenchel--Young loss in \cref{prop:fyloss-properties}, i.e., $g_t = \hat x_t - x_t$. 
With this in mind, we can read \eqref{eq:uniform-convexity-inequality} as 
$\norm{g_t}^2 \lesssim \inpr{g_t, \hat c_t - c^*}/\Delta^2$ and hence 
\begin{equation}
  R_T 
  =
  \sum_{t=1}^T \inpr{g_t, \hat c_t - c^*} 
  \lesssim 
  \sqrt{\sum_{t=1}^T\norm{g_t}^2}
  \lesssim 
  \sqrt{\frac{R_T}{\Delta^2}}.
\end{equation}
Consequently, $R_T$ is bounded by a term of lower order in $R_T$ itself, obtaining $R_T \lesssim 1/\Delta^2$. 
This proof strategy, called the \emph{self-bounding technique}, is a powerful tool recently popularized in online learning \citep{Gaillard2014-hn,Zimmert2021-kp}, and the Fenchel--Young-loss perspective clarifies how it can be adopted.

\subsection{Discussion on the $\Delta$-Gap Condition}\label{subsec:applications}
We discuss situations where the $\Delta$-gap condition in \cref{def:gap} is satisfied.
Our $\Delta$-gap condition is inspired by \citet{Weed2018-ii}, who analyzed error bounds in linear programs (LPs) with the entropic penalty. 
In \citet{Weed2018-ii}, the feasible region, $X_t$ in our notation, is the set of all vertices of a polytope specified by LP constraints. 
For such $X_t$, \citet{Weed2018-ii} assumed that the optimal objective value is better than the others by at least $\Delta' > 0$ 
(i.e., for any $x' \in X_t$, $\max_{x \in X_t}\inpr{c^*, x} - \inpr{c^*, x'}$ is zero or at least $\Delta'$).
Our $\Delta$-gap condition is slightly more demanding in that it implicitly requires the uniqueness of the optimal solution. 
Still, if the assumption of \citet{Weed2018-ii} holds with $\Delta'$ and the optimal solution $x$ for $c^*$ is unique, our $\Delta$-gap condition holds with $\Delta = \Delta'/K$, where $K$ is the diameter of $X_t$ with respect to $\norm{\cdot}$.

As discussed in \citet{Weed2018-ii}, such a gap condition holds in LPs with integral polytopes and integral objectives, which often appear in combinatorial optimization, such as the shortest path and bipartite matching problems. 
Specifically, if $X_t \subseteq \Z^n$ is a set of integral vertices and $c^* \in \Z^n$ holds, there is a gap of $\Delta' \ge 1$ between optimal and suboptimal objective values. 
Similarly, a gap of $\Delta' \ge 1$ exists in integer LPs with $c^* \in \Z^n$. 
In such cases, if it additionally holds that the optimal solution $x$ is unique, our $\Delta$-gap condition holds with $\Delta \ge 1/K$.

The $\Delta$-gap condition can be satisfied even if $X_t$ and $c^*$ are not integral. 
Let $\norm{\cdot}$ and $\norm{\cdot}_\star$ be the $\ell_2$-norm, $X_t$ a polytope, and $x$ a vertex feasible solution, as in \cref{fig:gap-condition}. 
Let $N(x) \coloneqq \Set*{c \in \R^n}{\inpr{c, x - \hat x} \ge 0,\, \forall \hat x \in X_t}$ denote the normal cone at $x$.
If $c^*$ lies in the interior of $N(x)$, then $x$ is the unique optimal solution for $c^*$.
The $\Delta$-gap condition, $\inpr{c^*, \frac{x - \hat x}{\norm{x - \hat x}}} \ge \Delta$ for all $\hat x \in X_t$, means that the cosine of the angle between $c^*$ and $x - \hat x$ is at least $\Delta/\norm{c^*}$.
A sufficient condition for this is that the distance between $c^*$ and the closest boundary of $N(x)$ is at least $\Delta$, as in \citet[Lemma~18]{Sakaue2024-mu}. 
More precisely, their lemma is concerned with the distance from $c^*$ to the \emph{frontier} defined as follows:
\[
  F \coloneqq \set*{
    c \in \mathbb{R}^n 
  \;\colon\;
    \abs*{
      \argmax_{x \in \mathcal{Y}} \inpr{c, x}
    } \ge 2
  },
\]
where $\mathcal{Y}$ is the set of vertices of $X_t$.
The is nothing but the boundary of normal cones, $\bigcup_{x \in \mathcal{Y}}\mathrm{bd}(N(x))$, since 
\begin{align}
  c \in F 
  \iff{}& \exists x, x' \in \mathcal{Y} \text{ s.t. } x \neq x', \forall \hat x \in \mathcal{Y}, \inpr{c, x} = \inpr{c, x'} \ge \inpr{c, \hat x}
  \\
  \iff{}& \exists x, x' \in \mathcal{Y} \text{ s.t. } x \neq x', \forall \hat x \in X_t, \inpr{c, x - x'} = 0 \text{ and } \inpr{c, x} \ge \inpr{c, \hat x}
  \\
  \iff{}& \exists x \in \mathcal{Y}, c \in \mathrm{bd}(N(x))
\end{align}
holds, where we used $X_t = \mathrm{conv}(\mathcal{Y})$ and the fact that a (bounded and feasible) LP always has a vertex optimal solution. 
Therefore, we have $F = \bigcup_{x \in \mathcal{Y}} \mathrm{bd}(N(x))$. 
Consequently, \citet[Lemma~18]{Sakaue2024-mu} ensures that the $\Delta$-gap condition holds if $c^*$ is distant from boundaries of all $N(x)$ by at least $\Delta$.
This consequence is intuitive:
$c^*$ close to the boundary of $N(x)$ is ambiguous in that slightly perturbing $c^*$ can lead to a different optimal solution, and the $O(1/\Delta^2)$ regret bound in \cref{thm:fast-rate} means that less ambiguous $c^*$ is easier to infer.

\begin{figure}
  \centering
  \resizebox{0.375\textwidth}{!}{
  \begin{tikzpicture}
    \coordinate (O) at (3, 0); 

    \draw[thick, darkgray] (O) -- ++(160:6) coordinate (A);
    \draw[thick, darkgray] (O) -- ++(240:2) coordinate (B); 
    \coordinate (C) at (A |- B);
    \fill[gray!20] (O) -- (A) -- (C) -- (B) -- cycle;

    \draw[thick, blue] (O) -- ++(70:4) coordinate (E);
    \draw[thick, blue] (O) -- ++(330:3) coordinate (F);
    \fill[blue!10] (O) -- (E) -- (F) -- cycle;

    \coordinate (OA) at ($(O)!0.25cm!(A)$);
    \coordinate (OE) at ($(O)!0.25cm!(E)$);
    \coordinate (R) at ($(OA) + (OE) - (O)$);
    \draw[black] (OA) -- (R) -- (OE);

    \coordinate (OB) at ($(O)!0.25cm!(B)$);
    \coordinate (OF) at ($(O)!0.25cm!(F)$);
    \coordinate (S) at ($(OB) + (OF) - (O)$);
    \draw[black] (OB) -- (S) -- (OF);

    \draw[line width=2pt, darkgray, dashed] (O) -- ++(170:4) coordinate (G); 
    \filldraw[darkgray] (G) circle (2.5pt);

    \draw[line width=2pt, ->, darkgray] (O) -- ++(-10:1.5) coordinate (H); 
    \draw[line width=2pt, ->, blue] (O) -- ++(20:2) coordinate (I);

    \filldraw[darkgray] (O) circle (2.5pt);

    \node[below, yshift=-10pt] at (O) {\Large \textcolor{black}{$x$}}; 
    \node[below] at (G) {\Large \textcolor{black}{$\hat x$}};
    \node[below, yshift=-35pt] at (G) {\Large \textcolor{black}{$X_t$}};  

    \node[right, yshift=0pt] at (H) {\LARGE \textcolor{black}{$\frac{x - \hat x}{\norm{x - \hat x}}$}}; 
    \node[right, yshift=7pt] at (I) {\Large \textcolor{blue}{$c^*$}};
    \node[xshift=-18pt, yshift=25pt] at (I) {\Large \textcolor{blue}{$N(x)$}};
  \end{tikzpicture}
  }
  \caption{
    An illustration of the gap condition. 
    The gray area shows polyhedral feasible region $X_t$. 
    The vertex $x$ is the unique optimal solution for $c^*$ if $c^*$ lies in the interior of the normal cone $N(x)$, shown in blue.
    The $\Delta$-gap condition requires that the cosine of the angle between $c^*$ and $x - \hat x$ is at least $\Delta/\norm{c^*}$ for every $\hat x \in X_t$; this is true if $c^*$ (the head of the blue arrow) is distant from the boundary of $N(x)$ (the blue lines) by at least $\Delta$.
  }
  \label{fig:gap-condition}
\end{figure}
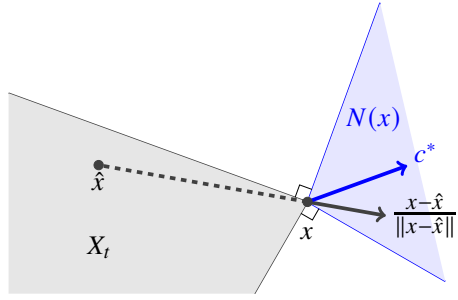

\section{Conclusion}
This paper has revisited the online learning approach to inverse linear optimization. 
We have shown that the \subloss can be seen as a Fenchel--Young loss and that its linearized regret equals the \totloss.
As a byproduct, we have obtained an offline guarantee on the \subloss without assuming the optimality of the agent's choices.
We have also obtained a bound on the cumulative \totloss that is independent of $T$ by exploiting the gap condition on the agent's decision problems.
We believe that our Fenchel--Young loss perspective, which connects inverse linear optimization and OCO, will be beneficial in synthesizing the research streams and paving the way for further investigations.

There are several directions for future work.
For example, analyzing what if we employ other $\Omega$ than the indicator function is interesting from the Fenchel--Young loss perspective.
Investigating the relation to the \emph{Fitzpatrick loss} \citep{Rakotomandimby2024-ac}, a recent framework for designing tighter surrogate losses than Fenchel--Young losses, will also be interesting.
Revealing the lower bound on the cumulative \totloss is also an important problem left for future work.

\subsubsection*{Acknowledgements}
The authors are grateful to anonymous reviewers for their constructive feedback.
SS is supported by JST ERATO Grant Number JPMJER1903.  
HB is supported by JST PRESTO Grant Number JPMJPR24K6.
TT is supported by JST ACT-X Grant Number JPMJAX210E and JSPS KAKENHI Grant Number JP24K23852.

\newrefcontext[sorting=nyt]
\printbibliography[heading=bibintoc]

\appendix
\clearpage
\onecolumn

\section{Proof of \texorpdfstring{\cref{prop:omd-regret-bound}}{Proposition~\ref{prop:omd-regret-bound}}}\label{asec:proof}
This section presents the proof of \cref{prop:omd-regret-bound}.
The following proof is a slight generalization of the analysis of AdaHedge in \citet[Section~7.6]{orabona2023modern}, which we include here for completeness.

\begin{proof}
Based on the notation used in \citet[Section~7]{orabona2023modern}, we define $\psi_t(c) \coloneqq \beta_t(\psi(c) - \min_{c' \in \Theta}\psi(c'))$ and $F_t(c) \coloneqq \psi_t(c) + \sum_{i=1}^{t-1} \inpr{g_i, c}$ for $t = 1,\dots,T$.
Note that replacing the regularizer $\beta_t\psi$ with $\psi_t$ in \cref{alg:omd} does not affect the choice of $\hat c_t$, and hence $\hat c_t \in \argmin_{c \in \Theta} F_t(c)$.
Also, $\psi_{t+1}(c) \ge \psi_t(c)$ holds since $\beta_t$ is non-decreasing.

From \citet[Lemma~7.1]{orabona2023modern}, we have 
\begin{align}
  \begin{aligned}\label{eq:appa}
  R_T 
  ={}& \psi_{T+1}(c^*) - \underbrace{\min_{c\in \Theta}\psi_1(c) }_{=0} +  \sum_{t=1}^T\prn*{F_t(\hat c_t) - F_{t+1}(\hat c_{t+1}) + \inpr{g_t, \hat c_t}} + \underbrace{F_{T+1}(\hat c_{T+1}) - F_{T+1}(c^*)}_{\le 0}
  \\
  \le{}& \beta_{T+1}\prn*{\psi(c^*) - \min_{c \in \Theta}\psi(c)} +  \sum_{t=1}^T\prn*{F_t(\hat c_t) - F_{t+1}(\hat c_{t+1}) + \inpr{g_t, \hat c_t}}
  \\
  \le{}& \beta_{T+1} B^2 + \sum_{t=1}^T\prn*{F_t(\hat c_t) - F_{t+1}(\hat c_{t+1}) + \inpr{g_t, \hat c_t}}.
  \end{aligned}
\end{align}
Below, we derive an upper bound on the second term on the right-hand side.

For now, let $\beta_t = \frac1\alpha\sqrt{\sum_{i=1}^{t-1}\norm{g_i}^2}$, where $\alpha > 0$ is tuned later, and let $D > 0$ denote the diameter of $\Theta$ with respect to $\norm{\cdot}_\star$.
Since $\psi_t$ is $\lambda\beta_t$-strongly convex, \citet[Lemma~7.8]{orabona2023modern} implies
\begin{equation}\label{eq:appc}
  F_t(\hat c_t) - F_{t+1}(\hat c_{t+1}) + \inpr{g_t, \hat c_t} 
  \le \frac{\norm{g_t}^2}{2\lambda\beta_t} + \underbrace{\psi_t(\hat c_{t+1}) -  \psi_{t+1}(\hat c_{t+1})}_{\le 0} 
  \le \frac{\alpha\norm{g_t}^2}{2\lambda\sqrt{\sum_{i=1}^{t-1}\norm{g_i}^2}},
\end{equation}
where the right-hand side is infinite if $\sum_{i=1}^{t-1}\norm{g_i}^2 = 0$.
Also, as in \citet[Section~7.6]{orabona2023modern}, we have 
\begin{align}
  F_t(\hat c_t) - F_{t+1}(\hat c_{t+1}) + \inpr{g_t, \hat c_t} 
  \le{}& 
  F_t(\hat c_{t+1}) - F_{t+1}(\hat c_{t+1}) + \inpr{g_t, \hat c_t}
  \\
  ={}& 
  \psi_t(\hat c_{t+1}) + \sum_{i=1}^{t-1} \inpr{g_i, \hat c_{t+1}}
  - \prn*{\psi_{t+1}(\hat c_{t+1}) + \sum_{i=1}^{t} \inpr{g_i, \hat c_{t+1}}}
  + \inpr{g_t, \hat c_t}
  \\
  ={}&
  \underbrace{\psi_t(\hat c_{t+1}) - \psi_{t+1}(\hat c_{t+1})}_{\le 0}
  -\inpr{g_t, \hat c_{t+1}} + \inpr{g_t, \hat c_t}  
  \\
  \le{}& 
  -\inpr{g_t, \hat c_{t+1}} + \inpr{g_t, \hat c_t} 
  \\
  \le{}& 
  D\norm{g_t}.
\end{align}
These two inequalities imply
\begin{align}
  \sum_{t=1}^T\prn*{F_t(\hat c_t) - F_{t+1}(\hat c_{t+1}) + \inpr{g_t, \hat c_t}}
  \le{}&
  \sum_{t=1}^T \min\set*{
    \frac{\alpha\norm{g_t}^2}{2\lambda\sqrt{\sum_{i=1}^{t-1}\norm{g_i}^2}},
    D\norm{g_t}
  }
  \\
  ={}&
  \sum_{t=1}^T \sqrt{
    \min\set*{
    \frac{\alpha^2\norm{g_t}^4}{4\lambda^2\sum_{i=1}^{t-1}\norm{g_i}^2},
    D^2\norm{g_t}^2
  }
  }
  \\
  \le{}&
  \sum_{t=1}^T \sqrt{
    \frac2{
      \frac{4\lambda^2\sum_{i=1}^{t-1}\norm{g_i}^2}{\alpha^2\norm{g_t}^4}
      +
      \frac1{D^2\norm{g_t}^2}
    }
  }
  \\
  ={}&
  \sum_{t=1}^T 
  \frac{\sqrt{2} \alpha D \norm{g_t}^2}{
      \sqrt{
        4\lambda^2D^2\sum_{i=1}^{t-1}\norm{g_i}^2
        +
        \alpha^2\norm{g_t}^2
      }
    },  
\end{align}
where the second inequality uses the fact that the minimum of positive numbers is at most their harmonic mean.
Therefore, if $\alpha \ge 2\lambda D$, which we confirm shortly, we have
\begin{align}\label{eq:appb}
  \sum_{t=1}^T\prn*{F_t(\hat c_t) - F_{t+1}(\hat c_{t+1}) + \inpr{g_t, \hat c_t}}
  \le
  \frac{\sqrt{2}\alpha}{2\lambda}
  \sum_{t=1}^T 
  \frac{\norm{g_t}^2}{
      \sqrt{
        \sum_{i=1}^{t}\norm{g_i}^2
      }
    }
  \le
  \frac{\sqrt{2}\alpha}{\lambda}
  \sqrt{\sum_{t=1}^{T}\norm{g_t}^2},  
\end{align}
where the last inequality is due to \citet[Lemma~4.13]{orabona2023modern}.

Consequently, from \eqref{eq:appa} and \eqref{eq:appb}, $R_T$ is bounded as
\[
  R_T 
  \le 
  \prn*{
    \frac{B^2}{\alpha}  
    + 
    \frac{\sqrt{2}\alpha}{\lambda}
  }
  \sqrt{\sum_{t=1}^{T}\norm{g_t}^2}.
\]
The right-hand side is minimized when $\alpha = \frac{B\sqrt{\lambda}}{2^{1/4}}$, which satisfies $\alpha \ge 2\lambda D$ due to the assumption of $B^2 \ge 2^{5/2}\lambda D^2$. 
Therefore, \cref{alg:omd} with 
$
  \beta_t = \frac{2^{1/4}}{B}\sqrt{\frac{\sum_{i=1}^{t-1} \norm{g_i}^2}{\lambda}}
$
attains $R_T \le 2^{5/4}B \sqrt{\frac{1}{\lambda}\sum_{t=1}^{T}\norm{g_t}^2}$.
\end{proof}

\section{Recovering the Same Bound as \citet[Theorem~3.3]{Barmann2017-wl}}\label{asec:barmann}
As with \citet[Theorem~3.3]{Barmann2017-wl}, we assume that the diameter of $X_t$'s with respect to $\norm{\cdot} = \norm{\cdot}_\infty$ is at most $K > 0$.
We define $H >0$ as a constant such that $\psi(c^*) - \min_{c \in \Theta}\psi(c) \le H^2$. 
We consider using FTRL (\cref{alg:omd}) with the following choice of $\beta_t$: 
\begin{equation}
  \beta_t = \frac{1}{H\sqrt{\lambda}}\sqrt{K^2 + \sum_{i=1}^{t-1}\norm{g_i}^2},
\end{equation}
which is also non-decreasing. 
We let $\beta_{T+1} = \beta_T$, which does not affect the analysis \citep[Remark~7.3]{orabona2023modern}.

From \eqref{eq:appa}, \eqref{eq:appc}, and $\beta_{T+1} =\beta_T \le \frac{K}{H}\sqrt{\frac{T}{\lambda}}$, which is due to $\norm{g_i} = \norm{\hat x_i - x_i} \le K$, we have 
\begin{equation}
  R_T 
  \overset{\text{\eqref{eq:appa}}}{\le} 
  \beta_{T+1} H^2 + \sum_{t=1}^T\prn*{F_t(\hat c_t) - F_{t+1}(\hat c_{t+1}) + \inpr{g_t, \hat c_t}}
  \overset{\text{\eqref{eq:appc}}}{\le}
  \beta_{T+1} H^2 + \sum_{t=1}^T\frac{\norm{g_t}^2}{2\lambda\beta_t}
  \le
  KH\sqrt{\frac{T}{\lambda}} + \sum_{t=1}^T\frac{\norm{g_t}^2}{2\lambda\beta_t}.
\end{equation}
The second term on the right-hand side is bounded as
\[
  \sum_{t=1}^T\frac{\norm{g_t}^2}{2\lambda\beta_t}
  = 
  \sum_{t=1}^T\frac{\norm{g_t}^2}{\frac{2\sqrt{\lambda}}{H}\sqrt{K^2 + \sum_{i=1}^{t-1}\norm{g_i}^2}}
  \le
  \frac{H}{2\sqrt{\lambda}}\sum_{t=1}^T\frac{\norm{g_t}^2}{\sqrt{\sum_{i=1}^{t}\norm{g_i}^2}}
  \le
  \frac{H}{\sqrt{\lambda}}\sqrt{\sum_{t=1}^{T}\norm{g_t}^2}
  \le 
  KH\sqrt{\frac{T}{\lambda}},
\]
where we used $\norm{g_t} = \norm{\hat x_t - x_t} \le K$ in the first and last inequalities, and \citet[Lemma~4.13]{orabona2023modern} in the second inequality, as in \eqref{eq:appb}.
Consequently, we obtain $R_T \le 2KH\sqrt{\frac{T}{\lambda}}$. 

As discussed in \cref{subsec:omd}, if $\Theta$ is the probability simplex and $\psi\colon\Theta\to\R$ is the negative Shannon entropy, we have $H = \sqrt{\ln n}$ and $\lambda = 1$, recovering $R_T \le 2K\sqrt{T\ln n}$. 
This bound is exactly the same as that of \citet[Theorem~3.3]{Barmann2017-wl}, including the constant factor of $2$.

\section{Proof of \texorpdfstring{\cref{thm:online-to-batch}}{Theorem~\ref{thm:online-to-batch}}}\label{asec:online-to-batch}


\begin{proof}
  Since the \subloss $\ell^\mathrm{sub}$ is convex due to \cref{prop:fyloss-properties,prop:subloss-fyloss}, the claim follows from the standard online-to-batch conversion scheme (e.g., \citealt[Theorem~3.1]{orabona2023modern}), which we detail below for completeness.
  Since $\hat c_t$ is independent of $\set*{(X_i, x_i)}_{i\ge t}$,
  \begin{equation}
    \E\brc*{
      \ell^\mathrm{sub}_{X, x}(\hat c_t)
    }
    =
    \E\brc*{
    \E_{(X_t, x_t)\sim\Dcal}\Brc*{
    \ell^\mathrm{sub}_{(X_t, x_t)}(\hat c_t)
    }{
      \set*{(X_i, x_i)}_{i=1}^{t-1}
    } 
    }
    =
    \E\brc*{
    \ell^\mathrm{sub}_t(\hat c_t)
    } 
  \end{equation}
  holds by the law of total expectation.
  By using this, together with $\E\brc*{\ell^\mathrm{sub}_{X,x}(\hat c)} \le \E\brc*{\frac1T\sum_{t=1}^T \ell^\mathrm{sub}_{X,x}(\hat c_t)}$, which follows from Jensen's inequality, and 
  $\E_{(X, x)\sim\Dcal}[\ell^\mathrm{sub}_{(X, x)}(c^*)] = \E[\ell^\mathrm{sub}_t(c^*)]$, we obtain 
  \begin{equation}
  \E\brc*{
    \ell^\mathrm{sub}_{X, x}(\hat c)
  }
  -
  \E_{(X, x) \sim \Dcal}\brc*{
    \ell^\mathrm{sub}_{X, x}(c^*)
  }
  \le{}
  \E\brc*{
    \frac{1}{T}\sum_{t=1}^T \prn*{\ell^\mathrm{sub}_t(\hat c_t) - \ell^\mathrm{sub}_t(c^*)}
  }
  =
  \E\brc*{
    \frac{R^\mathrm{sub}_T}{T}
  }.   
  \end{equation}
  The bound with \cref{alg:omd} follows from $R^\mathrm{sub}_T \le R_T$ due to \cref{cor:subopt}. 
\end{proof}

\end{document}